\documentclass{article}

\usepackage[utf8]{inputenc} % allow utf-8 input
\usepackage[T1]{fontenc}    % use 8-bit T1 fonts
\usepackage{hyperref}       % hyperlinks
\usepackage{url}            % simple URL typesetting
\usepackage{booktabs}       % professional-quality
\usepackage{amsfonts}    % blackboard math symbols
\usepackage{color}
\usepackage{epsfig}
\usepackage{algorithm}
\usepackage{algorithmic}
\usepackage{amsthm}
\usepackage{amsmath}
\usepackage{float}
\usepackage{graphicx}
\usepackage{wrapfig}
\usepackage{makecell}
\usepackage{caption}
\usepackage{subcaption}
\usepackage{bm}
\usepackage{multirow}

\newcommand{\nop}[1]{}
\DeclareMathOperator*{\argmax}{arg\,max}
\DeclareMathOperator*{\argmin}{arg\,min}

\newtheorem{definition}{Definition}
\newtheorem{lemma}{Lemma}

\newtheorem{proposition}{Proposition}

\usepackage[accepted]{icml2018}

\begin{document}
	
\twocolumn[
\icmltitle{Learning K-way D-dimensional Discrete Codes for Compact Embedding Representations}

\begin{icmlauthorlist}
	\icmlauthor{Ting Chen}{ucla}
    \icmlauthor{Martin Renqiang Min}{nec}
    \icmlauthor{Yizhou Sun}{ucla}
\end{icmlauthorlist}

\icmlaffiliation{ucla}{Department of Computer Science, University of California, Los Angeles}
\icmlaffiliation{nec}{NEC Laboratories America}

\icmlcorrespondingauthor{Ting Chen}{tingchen@cs.ucla.edu}
\icmlcorrespondingauthor{Martin Renqiang Min}{renqiang@nec-labs.com}
\icmlcorrespondingauthor{Yizhou Sun}{yzsun@cs.ucla.edu}

\icmlkeywords{KD Code, Representation Learning, Embedding, Model Compression}

\vskip 0.3in
]
\printAffiliationsAndNotice{}

\begin{abstract}	
%Embedding methods such as word embedding and graph embedding have become pillars for many applications containing discrete structures. 
Conventional embedding methods directly associate each symbol with a continuous embedding vector, which is equivalent to applying a linear transformation based on a ``one-hot'' encoding of the discrete symbols. Despite its simplicity, such approach yields the number of parameters that grows linearly with the vocabulary size and can lead to overfitting. In this work, we propose a much more compact K-way D-dimensional discrete encoding scheme to replace the ``one-hot" encoding. In the proposed ``KD encoding'', each symbol is represented by a $D$-dimensional code with a cardinality of $K$, and the final symbol embedding vector is generated by composing the code embedding vectors. To end-to-end learn semantically meaningful codes, we derive a relaxed discrete optimization approach based on stochastic gradient descent, which can be generally applied to any differentiable computational graph with an embedding layer. In our experiments with various applications from natural language processing to graph convolutional networks, the total size of the embedding layer can be reduced up to 98\% while achieving similar or better performance.
\end{abstract}

\vspace{-2em}
\section{Introduction}
Embedding methods, such as word embedding \cite{mikolov2013distributed,pennington2014glove}, have become pillars in many applications when learning from discrete structures. The examples include language modeling \cite{kim2016character}, machine translation \cite{sennrich2015neural}, text classification \cite{zhang2015character}, knowledge graph and social network modeling \cite{bordes2013translating,chen2017task}, and many others \cite{kipf2016semi,chen2016entity}. The objective of the embedding module in neural networks is to represent a discrete symbol, such as a word or an entity, with some continuous embedding vector $\bm{v}\in R^{d}$. This seems to be a trivial problem, at the first glance, in which we can directly associate each symbol with a learnable embedding vector, as is done in existing work. To retrieve the embedding vector of a specific symbol, an embedding table lookup operation can be performed. This is equivalent to the following: first we encode each symbol with an ``one-hot'' encoding vector $\bm{b} \in [0,1]^N$ where $\sum_j \bm{b}_j=1$ ($N$ is the total number of symbols), and then generate the embedding vector $\bm v$ by simply multiplying the ``one-hot'' vector $\bm{b}$ with the embedding matrix $W\in R^{N\times d}$, i.e. $\bm{v}=W^T\bm{b}$.

Despite the simplicity of this ``one-hot'' encoding based embedding approach, it has several issues. The major issue is that the number of parameters grows linearly with the number of symbols. This becomes very challenging when we have millions or billions of entities in the database, or when there are lots of symbols with only a few observations (e.g. Zipf's law). There also exists redundancy in the $O(N)$ parameterization, considering that many symbols are actually similar to each other. This over-parameterization can further lead to overfitting; and it also requires a lot of memory, which prevents the model from being deployed to mobile devices. Another issue is purely from the code space utilization perspective, where we find ``one-hot'' encoding is extremely inefficient. Its code space utilization rate is almost zero as $N / 2^N \rightarrow 0$ when $N \rightarrow \infty$, while $N$ dimensional discrete coding system can effectively represent $2^N$ symbols.

To address these issues, we propose a novel and much more compact coding scheme that replaces the ``one-hot'' encoding. In the proposed approach, we use a $K$-way $D$-dimensional code to represent each symbol, where each code has $D$ dimensions, and each dimension has a cardinality of $K$. For example, a concept of cat may be encoded as (5-1-3-7), and a concept of dog may be encoded as (5-1-3-9). The code allocation for each symbol is based on data and specific tasks such that the codes can capture semantics of symbols, and \textit{similar codes} should reflect \textit{similar meanings}. While we mainly focus on the encoding of symbols in this work, the learned discrete codes can have larger applications, such as information retrieval. We dub the proposed encoding scheme as ``\textit{KD encoding}''.

The KD code system is much more compact than its ``one-hot'' counterpart. To represent a set of symbols of size $N$, the ``KD encoding'' only requires $K^D\ge N$. Increasing $K$ or $D$ by a small amount, we can easily achieve $K^D \gg N$, in which case it will still be much more compact and keep $D = O(\frac{\log N}{\log K})$. Consider $K=2$, the utilization rate of ``KD encoding'' is $N/2^D$, which is $2^{N-D}$ times more compact than its ``one-hot'' counterpart\footnote{Assuming we have vocabulary size $N=10,000$ and the dimensionality $D=100$, it is $2^{9900}$ times more efficient.}.

The compactness of the code can be translated into compactness of the parametrization. Dropping the giant embedding matrix $W\in R^{N\times d}$ that stores symbol embeddings and leveraging semantic similarities between symbols, the symbol embedding vector is generated by composing much fewer code embedding vectors. This can be achieved as follows: first we embed each ``KD code'' into a sequence of code embedding vectors in $R^{D\times d'}$, and then apply embedding transformation function $\bm{f}(\cdot)$ to generate the final symbol embedding. By adopting the new approach, we can reduce the number of embedding parameters from $O(Nd)$ to $O(\frac{K}{\log K} d' \log N + C)$, where $d'$ is the code embedding size, and $C$ is the number of neural network parameters.

Due to the the discreteness of the code allocation problem, it is very challenging to learn the meaningful discrete codes that can exploit the similarities among symbols according to a target task in an end-to-end fashion. A compromise is to learn the code given a trained embedding matrix, and then fix the code during the stage of task-specific training. While this has been shown working relatively well in previous work \cite{chen2017learning,shu2017compress}, it produces a sub-optimal solution, and requires a multi-stage procedure that is hard to tune. In this work, we derive a relaxed discrete optimization approach based on stochastic gradient descent (SGD), and propose two guided methods to assist the end-to-end code learning. To validate our idea, we conduct experiments on three different tasks from natural language processing to graph convolutional networks for semi-supervised node classification. We achieve 95\% of embedding model size reduction in the language modeling task and 98\% in text classification with similar or better performance.
\section{The K-way D-dimensional Discrete Encoding Framework}
In this section, we introduce the ``KD encoding'' framework in details.
\subsection{Problem Formulation}
Symbols are represented with a vocabulary $V = \{s_1, s_2, \cdots, s_N\}$ where $s_i$ corresponds to the $i$-th symbol. Here we aim to learn a transformation function that maps a symbol $s_i$ to a continuous embedding vector $\bm{v}_i$, i.e. $\mathcal{T}: V \rightarrow \mathbb{R}^d$. In the case of conventional embedding method, $\mathcal{T}$ is a linear transformation of ``one-hot'' code of a symbol.

To measure the fitness of $\mathcal{T}$, we consider a differentiable computational graph $\mathcal{G}$ that takes discrete symbols as input $\bm{x}$ and outputs the predictions $\bm{y}$, such as text classification model based on word embeddings. We also assume a task-specific loss function $\mathcal{L}(\bm{x}, \bm{y})$ is given. The task-oriented learning of $\mathcal{T}$ is to learn $\mathcal{T}$ such that $\mathcal{L}(\bm{x}, \bm{y})$ is minimized, i.e. $\mathcal{T} = \argmin_{\mathcal{T}} \mathcal{L}(\bm{x}, \bm{y} | \mathcal{T}, \Theta)$ where $\Theta$ are task-specific parameters.
\subsection{The ``KD Encoding'' Framework}
In the proposed framework, each symbol is associated with a $K$-way $D$-dimensional discrete code. We denote the discrete code for the $i$-th symbol as $\bm{c}_i = (\bm{c}_i^1, \bm{c}_i^2, \cdots, \bm{c}_i^D) \in \mathcal{B}^D$, where $\mathcal{B}$ is the set of code bits with cardinality $K$. To connect symbols with discrete codes, a code allocation function $\bm\phi(\cdot): V \rightarrow \mathcal{B}^D$ is used. The learning of this mapping function will be introduced later, and once fixed it can be stored as a hash table for fast lookup. Since a discrete code $\bm c_i$ has $D$ dimensions, we do not directly use embedding lookup to find the symbol embedding as used in ``one-hot'' encoding. Instead, we want to learn an adaptive code composition function that takes a discrete code and generates a continuous embedding vector, i.e. $\bm{f}: \mathcal{B}^D\rightarrow R^d$. The details of $\bm f$ will be introduced in the next subsection. In sum, the ``KD encoding'' framework we have $\mathcal{T} = \bm f \circ \bm \phi$ with a ``KD code'' allocation function $\bm \phi$ and a composition function $\bm f$ as illustrated in Figure \ref{fig:framework}(a) and \ref{fig:framework}(b).

\begin{figure*}[]
	\begin{center}
		\includegraphics[trim=10 220 0 190,clip,width=\textwidth]{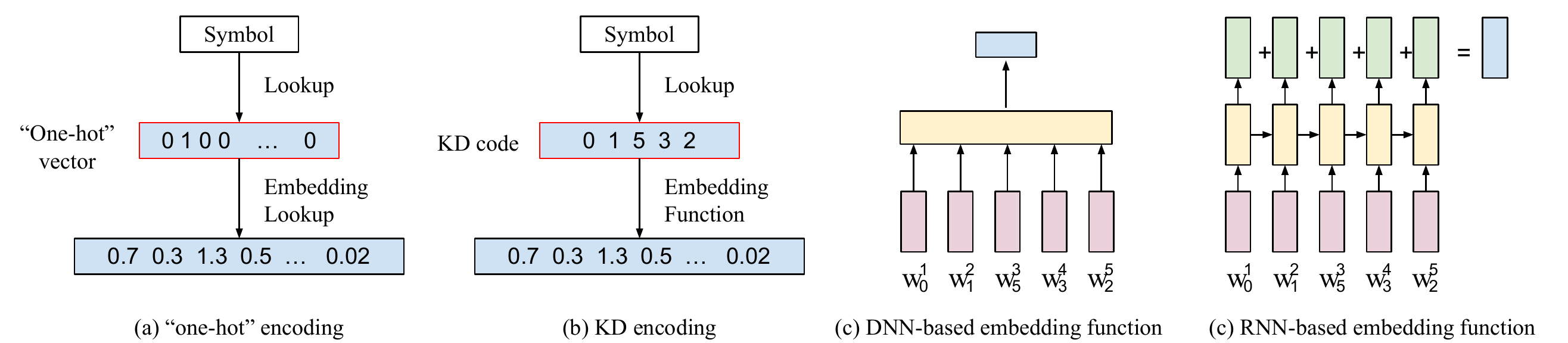}
	\end{center}
	\caption{\label{fig:framework} (a) The conventional symbol embedding based on ``one-hot'' encoding. (b) The proposed KD encoding scheme. (c) and (d) are examples of embedding transformation functions by DNN and RNN used in the ``KD encoding'' when generating the symbol embedding from its code.}
\end{figure*}

In order to uniquely identify every symbol, we only need to set $K^D = N$, as we can assign a unique code to each symbol in this case. When this holds, the code space is fully utilized, and none of the symbol can change its code without affecting other symbols. We call this type of code system \textit{compact code}. The optimization problem for compact code can be very difficult, and usually requires approximated combinatorial algorithms such as graph matching \cite{li2016lightrnn}. Realizing the difficulties in optimization, we propose to adopt the \textit{redundant code} system, where $K^D \gg N$, namely, there are a lot of ``empty'' codes with no symbol associated. Changing the code of one symbol may not affect other symbols under this scheme, since the random collision probability can be very small \footnote{For example, we can set $K=100, D=10$ for a billion symbols, in a random code assignment, the probability of the NO collision at all is 99.5\%.}, which makes it easier to optimize. The redundant code can be achieved by slightly increasing the size of $K$ or $D$ thanks to the exponential nature of their relations to $N$. Therefore, in both compact code or redundant code, it only requires $D = O(\frac{\log N}{\log K})$.

\subsection{Discrete Code Embedding}
As mentioned above, given learned $\bm\phi(\cdot)$ and the $i$-th symbol $s_i$, we can retrieve its code via a code lookup, i.e. $\bm{c}_i = \bm\phi(s_i)$. In order to generate the composite embedding vector $\bm{v}_i$, we adopt an adaptive code composition function $\bm v_i = \bm f(\bm c_i)$. To do so, we first embed the code $\bm{c}_i$ to a sequence of code embedding vectors $(\mathcal{W}^1_{\bm{c}_i^1}, \mathcal{W}^2_{\bm{c}_i^2}, \cdots, \mathcal{W}^D_{\bm{c}_i^D})$, and then apply another transformation $\bm v = \bm f_e(\mathcal{W}^1_{\bm{c}_i^1}, \mathcal{W}^2_{\bm{c}_i^2}, \cdots, \mathcal{W}^D_{\bm{c}_i^D}; \theta_e)$ to generate $\bm v$. Here $\mathcal{W}^j\in R^{K\times d'}$ is the code embedding matrix for the $j$-th code dimension, and $\bm f_e$ is the \textit{embedding transformation function} that maps the code embedding vectors to the symbol embedding vector. The choice of $f_e$ is very flexible and varies from task to task. In this work, we consider two types of embedding transformation functions. 

The first one is based on a linear transformation:
$$
\bm{v}_i = H\bigg(\sum_j \mathcal{W}^j_{\bm{c}_i^j}\bigg)^T, 
$$ 
where $H\in R^{d\times d'}$ is a transformation matrix for matching the dimensions. While this is simple and efficient, due to its linear nature, the capacity of the generated symbol embedding may be limited when the size of $K, D$ or the code embedding dimension $d'$ is small.

Another type of embedding transformation functions are nonlinear, and here we introduce one that is based on a recurrent neural network, LSTM \cite{hochreiter1997long}, in particular.  That is, we have $\bm (h_1, \cdots, h_j) =\text{LSTM}(\mathcal{W}^1_{\bm{c}^1}, \cdots,\mathcal{W}^j_{\bm{c}^j} )$ (see supplementary for details).

The final symbol embedding can be computed by summing over LSTM outputs at all code dimensions (and using a linear layer to match dimension if $d\ne d'$), i.e. $\bm{v} = H(\sum_j \bm{h}_j)^T$. Figure \ref{fig:framework}(c) and \ref{fig:framework}(d) illustrate the above two embedding transformation functions.

\subsection{Analysis of the Proposed ``KD Encoding''}
To measure the parameter and model size reduction, we first introduce two definitions as follows.
\begin{definition}
	(Embedding parameters) The embedding parameters are the parameters $\theta$ that are used in code composition function $\bm f$. Specifically, it includes code embedding matrices $\{\mathcal{W}\}$, as well as other parameters $\theta_e$ used in the embedding transformation function $\bm f_e$.
\end{definition}
It is worth noting that we do not explicitly include the code as embedding parameters. This is due to the fact that we do not count ``one-hot'' codes as parameters. Also in some cases the codes are not adaptively learned, such as hashed from symbols \cite{svenstrup2017hash}. However, when we export the model to embedded devices, the storage of discrete codes does occupy space. Hence, we introduce another concept below to take it into consideration as well.
\begin{definition}
	(Embedding layer's size) The embedding layer's size is the number of bits used to store both embedding parameters as well as the discrete codes.
\end{definition}
\begin{lemma}
	The number of embedding parameters used in KD encoding is $O(\frac{K}{\log K} d' \log N + C)$, where $C$ is the number of parameters of neural nets.
\end{lemma}
The proof is given in the supplementary material.

For the analysis of the embedding layer's size under ``KD encoding'', we assume that 32-bits floating point number is used. The total bits used by the ``KD encoding'' is $ND\log_2 K + 32(KDd' + C)$ consisting both code size as well as the size of embedding parameters. Comparing to the total model size by conventional full embedding, which is $32N(1+d)$, it can still be a huge saving of model space, especially when $N,d$ are large.

Here we provide a theoretical connection between the proposed ``KD encoding'' and the SVD or low-rank factorization of the embedding matrix. We consider the scenario where the composition function $\bm f$ is a linear function with no hidden layer, that is $\bm{v}_i = (\sum_j \mathcal{W}^j_{\bm{c}_i^j})^T$.
\begin{proposition}
	A linear composition function $\bm f$ with no hidden layer is equivalent to a sparse binary low-rank factorization of the embedding matrix.
\end{proposition}
The proof is also provided in the supplementary material. But the overall idea is that the ``$KD$ code'' mimics an 1-out-of-$K$ selection within each of the $D$ groups. 

The computation overhead brought by linear composition is very small compared to the downstream neural network computation (without hidden layer in linear composition function, we only need to sum up $D$ vectors). However, the expressiveness of the linear factorization is limited by the number of bases or rank of the factorization, which is determined by $K$ and $D$. And the use of non-linear composition function can largely increase the expressiveness of the composite embedding matrix and may be an appealing alternative, this is shown by the proposition 2 in supplementary.
\section{End-to-End Learning of the Discrete Code}

In this section, we propose methods for learning task-specific ``KD codes''.
\subsection{Continuous Relaxation for Discrete Code Learning}
As mentioned before, we want to learn the symbol-to-embedding-vector mapping function, $\mathcal{T}$, to minimize the target task loss, i.e. $\mathcal{T} = \argmin_{\mathcal{T}} \mathcal{L}(\bm{x}, \bm{y} | \mathcal{T}, \Theta)$. This includes optimizing both code allocation function $\bm \phi(\cdot)$ and code composition function $\bm f(\cdot)$. While $\bm f(\cdot)$ is differentiable w.r.t. its parameters $\theta$, $\bm \phi(\cdot)$ is very challenging to learn due to the discreteness and non-differentiability of the codes.

Specifically, we are interested in solving the following optimization problem,
\begin{equation}
\min_{\{\bm{c}\}, \theta, \Theta} \sum_i \mathcal{L} \bigg(\bm{x}_i, \bm{y}_i| \bm f_e\bigg(\mathcal{W}^1_{\bm{c}_i^1}, \mathcal{W}^2_{\bm{c}_i^2}, \cdots,\mathcal{W}^D_{\bm{c}_i^D}\bigg), \Theta \bigg)
\end{equation}
where $f_e$ is the embedding transformation function mapping code embedding to the symbol embedding, $\theta=\{\mathcal{W}, \theta_e\}$ contains code embeddings and the composition parameters, and $\Theta$ denotes other task-specific parameters.

We assume the above loss function is differentiable w.r.t. to the continuous parameters including embedding parameters $\theta$ and other task-specific parameters $\Theta$, so they can be optimized by following standard stochastic gradient descent and its variants \cite{kingma2014adam}. However, each $\bm{c}_i$ is a discrete code, it cannot be directly optimized via SGD as other parameters. In order to adopt gradient based approach to simplify the learning of discrete codes in an end-to-end fashion, we derive a continuous relaxation of the discrete code to approximate the gradient effectively. 

We start by making the observation that each code $\bm{c}_i$ can be seen as a concatenation of $D$ ``one-hot'' vectors, i.e. $\bm{c}_i = (\bm{o}^{1}_{i}, \bm{o}^{2}_i, \cdots, \bm{o}^{D}_i)$, where $\forall j, \bm{o}^j_i\in [0, 1]^K$ and $\sum_k \bm{o}^{jk}_i = 1$, where $\bm{o}^{jk}_i$ is the $k$-th component of $\bm{o}^{j}_i$. To make it differentiable, we relax the $\bm{o}_i$ from a ``one-hot'' vector to a continuous relaxed vector $\bm{\hat{o}}_i$ by applying \textit{tempering Softmax}:
$$
\bm{o}^{jk}_i \approx \bm{\hat{o}}^{jk}_i = \frac{\exp({\bm{\pi}}^{jk}_i/\tau)}{\sum_{k'} \exp({\bm{\pi}}^{jk'}_i/\tau)}
$$
Where $\tau$ is a temperature term, as $\tau\rightarrow 0$, this approximation becomes exact (except for the case of ties). We show this approximation effects for $K=2$ with $y = 1/(1+\exp(-x/\tau))$ in Figure \ref{fig:sigmoid_wtemp}. Similar techniques have been introduced in Gumbel-Softmax trick \cite{jang2016categorical,maddison2016concrete}.

\begin{figure}
	\centering
	\begin{subfigure}[b]{0.21\textwidth}
		\includegraphics[width=\textwidth]{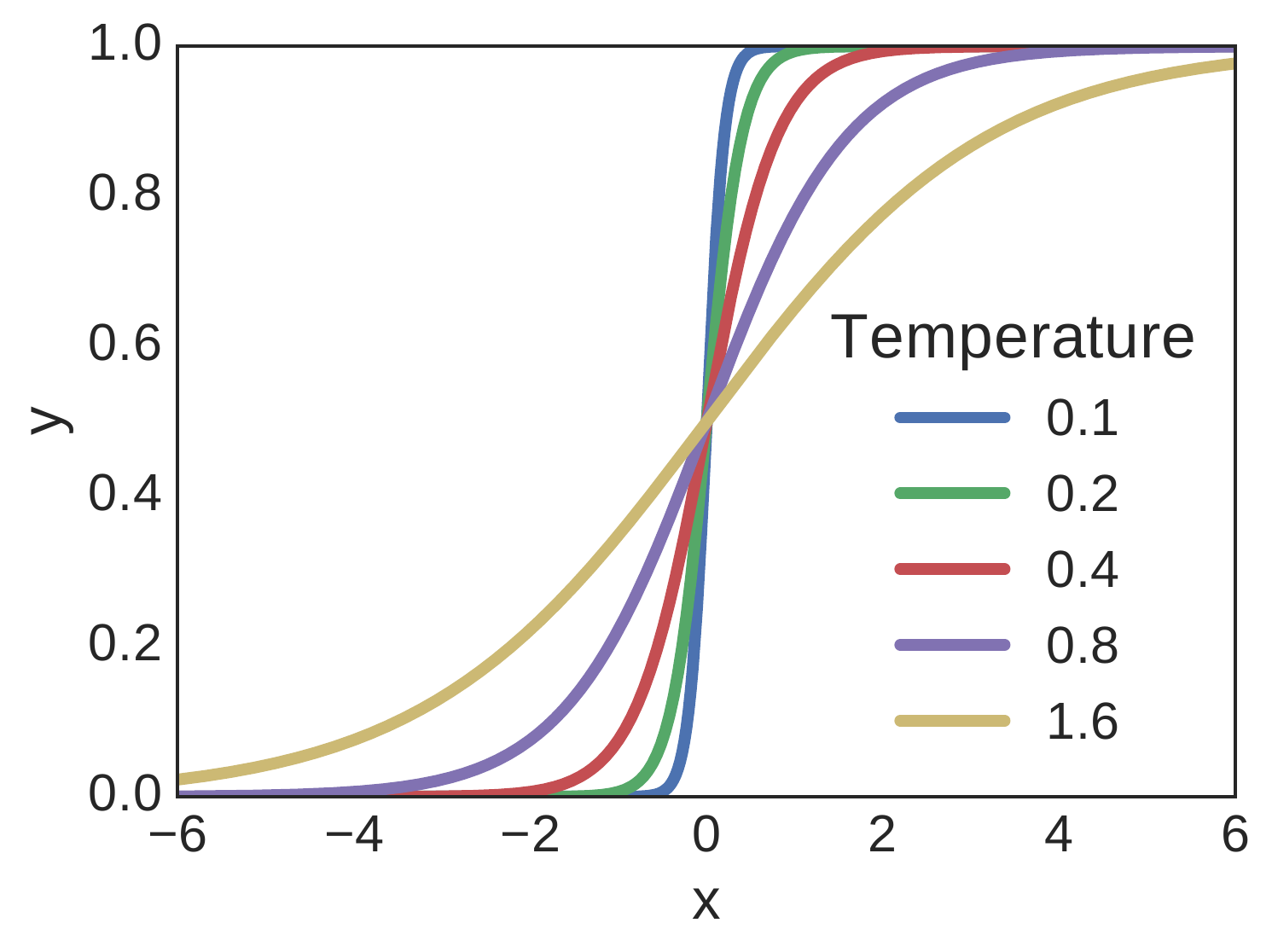}
		\caption{The output probabilities.}
		\label{fig:sigmoid_wtemp}
	\end{subfigure}
	\begin{subfigure}[b]{0.22\textwidth}
		\includegraphics[width=\textwidth]{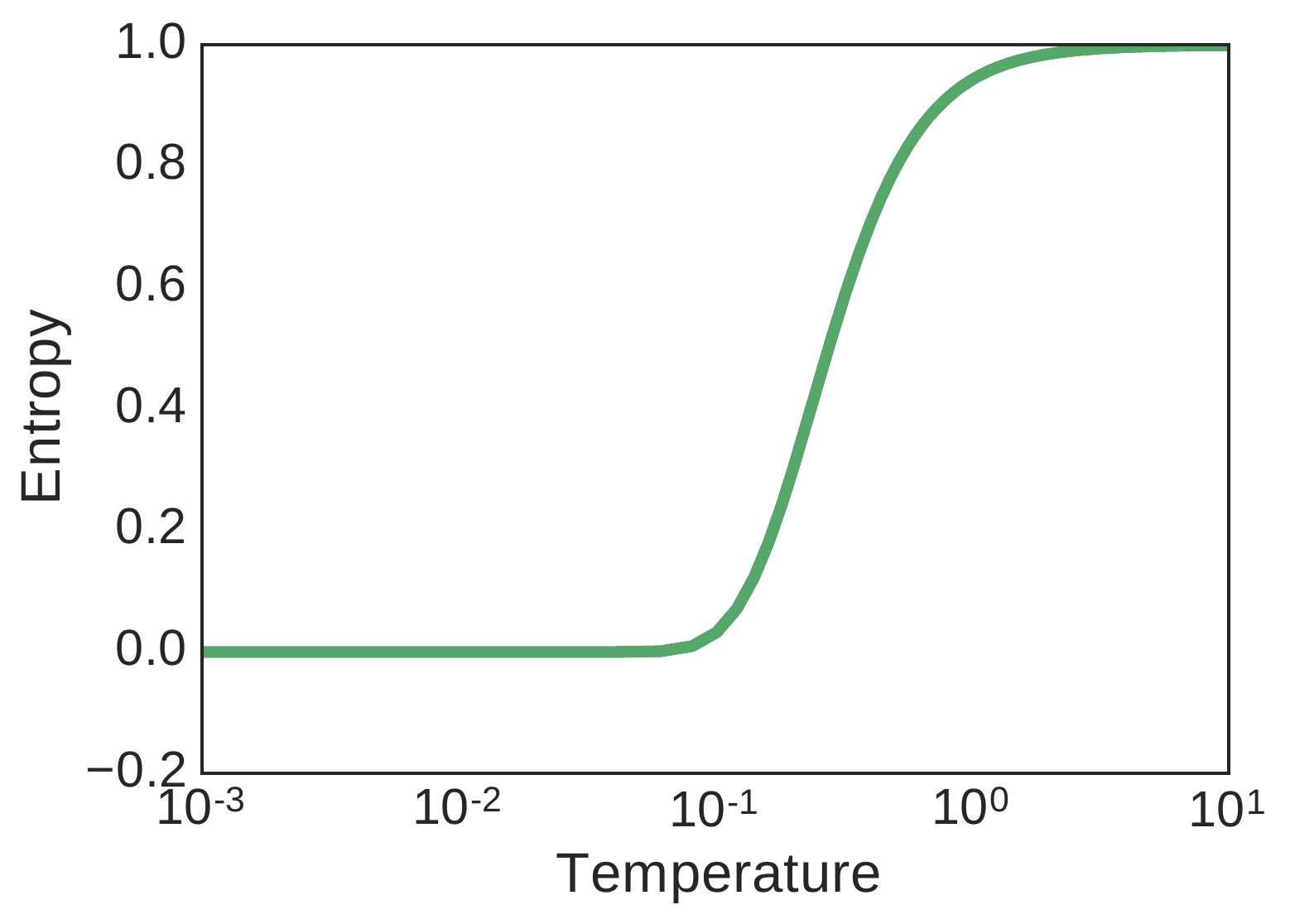}
		\caption{The entropy.}
		\label{fig:entropy_wtemp}
	\end{subfigure}
	\caption{The effects of temperature $\tau$ on output probability of Softmax and its entropy (when $K=2$). As $\tau$ decreases, the probabilistic output approximates step function when $K=2$, and generally ``one-hot'' vector when $K>2$.}\label{fig:wtemp}
\end{figure}

Since $\bm{\hat{o}}_i$ is continuous (given $\tau$ is not approaching 0), instead of learning the discrete code assignment directly, we learn $\bm{\hat{o}}_i$ as an approximation to ${\bm{o}}_i$. To do so, we can adjust the code logits $\bm \pi_i$ using SGD and gradually decrease the temperature $\tau$ during the training. Since the indexing operator for retrieval of code embedding vectors, i.e. $\mathcal{W}^j_{\bm{c}^j_i}$, is non-differentiable, to generate the embedding vector for $j$-th code dimension, we instead use an affine transformation operator, i.e. $(\mathcal{W}^j)^T \bm{\hat{o}}^j_i$, which enables the gradient to flow backwards normally.

It is easy to see that control of temperature $\tau$ can be important. When $\tau$ is too large, the output $\bm{\hat{o}}_i$ is close to uniform, which is too far away from the desired ``one-hot'' vector $\bm o_i$. When $\tau$ is too small, the slight differences between different logits $\bm \pi^j_i$ and $\bm \pi^{j'}_i$ will be largely magnified. Also, the gradient vanishes when the Softmax output approaches ``one-hot'' vector, i.e. when it is too confident. A ``right'' schedule of temperature can thus be crucial. While we can handcraft a good schedule of temperature, we also observe that the temperature $\tau$ is closely related to the entropy of the output probabilistic vector, as shown in Figure \ref{fig:entropy_wtemp}, where a same set of random logits can produce probabilities of different entropies when $\tau$ varies. This motivates us to implicitly control the temperature via regularizing the entropy of the model. To do so, we add the following entropy regularization term:
$
\mathbb{H} = -\sum_{i,j,k} \bm{\hat{o}}^{jk}_i\log \bm{\hat{o}}^{jk}.
$ A large penalty for this regularization term encourages a small entropy for the relaxed codes, i.e. a more spiky distribution.

Up to this point, we still use the continuous relaxation $\bm{\hat{o}}_i$ to approximate $\bm o_i$ during the training. In inference, we will only use discrete codes. The discrepancy of the continuous and discrete codes used in training and inference is undesirable. To close the gap, we take inspiration from Straight-Through Estimator \cite{bengio2013estimating}. In the forward pass, instead of using the relaxed tempering Softmax output $\bm{\hat{o}}_i$, which is likely a smooth continuous vector, we take its $\argmax$ and turn it into a ``one-hot'' vector as follows, which recovers a discrete code.
$$
\bm o^j_i = \text{one\_hot}\bigg(\argmax_k \bm{\hat{o}}^{jk}_i\bigg) \approx \text{Softmax}\bigg(\frac{\bm \pi^j_i}{\tau}\bigg)\text{,~~~}\tau\rightarrow 0
$$
We interpret the use of straight-through estimator as using different temperatures during the forward and backward pass. In forward pass, $\tau\rightarrow 0$ is used, for which we simply apply the $\argmax$ operator. In the backward pass (to compute the gradient), it pretends that a larger $\tau$ was used. Compared to using the same temperature in both passes, this always outputs ``one-hot'' discrete code $\bm{o}^j_i$, which closes the previous gap between training and inference.

The training procedure is summarized in Algorithm \ref{algo:STE_softmax}, in which the \texttt{stop\_gradient} operator will prevent the gradient from back-propagating through it.

\begin{algorithm}
	\small
	\caption{An epoch of code learning via Straight-through Estimator with Tempering Softmax.}
	\label{algo:STE_softmax}
	\begin{algorithmic}
	\STATE {\bfseries Parameters:} code logits $\{\bm \pi_i\}$, code embedding matrices $\{\mathcal{W}^j\}$, transformation parameters $\theta_e$, and other task specific parameters $\Theta$.
	\FOR{$i \gets 1$ \textbf{to} $N$}
		\FOR{$j \gets 1$ \textbf{to} $D$}
			\STATE $\bm{\hat{o}}^j_i = \text{Softmax}(\bm \pi^j_i/\tau)$
			\STATE $\bm{o}^j_i = \text{one\_hot}(\argmax_k \hat{\bm{o}}^{jk}_i)$
			\STATE $\bm{o}^j_i = \texttt{stop\_gradient}(\bm{o}^j_i - \bm{\hat{o}}^j_i) + \bm{\hat{o}}^j_i$
		\ENDFOR
		\STATE A step of SGD on $\bm \pi_i, \{\mathcal{W}^j\}, \theta_e, \Theta$ to reduce $\mathcal{L}\bigg(\bm{x}_i, \bm{y}_i, \bm f_e\bigg((\bm{o}^1_i)^T\mathcal{W}^1, \cdots, (\bm{o}^D_i)^T  \mathcal{W}^D; \theta_e\bigg), \Theta\bigg)$
	\ENDFOR
	\end{algorithmic}
\end{algorithm}

\subsection{Code Learning with Guidances}

It is not surprising the optimization problem is more challenging for learning discrete codes than learning conventional continuous embedding vectors, due to the discreteness of the problem (which can be NP-hard). This could lead to a suboptimal solution where discrete codes are not as competitive. Therefore, we propose to use  guidances from the continuous embedding vectors to mitigate the problem. The basic idea is that instead of adjusting codes according to noisy gradients from the end task as shown above, we also require the composite embedding vectors from codes to mimic continuous embedding vectors, which can be either jointly trained (online distillation guidance), or pre-trained (pre-train distillation guidance). The continuous embedding can provide better signals for both code learning as well as the rest parts of the neural network, improve the training subsequently.

\paragraph{Online Distillation Guidance (ODG).}
A good learning progress in code allocation function $\phi(\cdot)$ can be important for the rest of the neural network to learn. For example, it is hard to imagine we can train a good model based on ``KD codes'' if we have $\bm\phi(\text{``table''})=\bm\phi(\text{``cat''})$. However, the learning of the $\bm \phi(\cdot)$ also depends on the rest of network to provide good signals.

Based on the observation, we propose to associate a regular continuous embedding vector $\bm{u}_i$ with each symbol during the training, and we want the ``KD encoding'' function $\mathcal{T}(\cdot)$ to mimic the continuous embedding vectors, while both of them are simultaneously optimized for the end task. More specifically, during the training, instead of using the embedding vector generated from the code, i.e. $\bm{f}(\bm{c}_i)$, we use a dropout average of them, i.e. $$\bm{v}_i = m \odot \bm{u}_i + (1-m)\odot \bm{f}(\bm{c}_i).$$
Here $m$ is a Bernoulli random variable for selecting between the regular embedding vector or the KD embedding vector. When $m$ is turned on with a relatively high probability (e.g. 0.7), even if $\bm f(\bm{c}_i)$ is difficult to learn, $\bm{u}_i$ can still be learned to assist the improvement of the task-specific parameters $\Theta$, which in turn helps code learning. During the inference, we only use $\bm f(\bm{c}_i)$ as output embedding. This choice can lead to a gap between training and generalization errors. Hence, we add a regularization loss $\lambda\|\bm{u}_i-\bm{f}(\bm{c}_i)\|^2$ during the training that encourages the match between $\bm{u}_i$ and $\bm{f}(\bm{c}_i)$\footnote{Here we use $\text{stop\_gradient}(\bm{u}_i)$ to prevent embedding vectors $\bm u$ being dragged to $\bm{f}(\bm{c}_i)$ as it has too much freedom.}.% We call this technique \textit{Online Distillation Guidance}.

\paragraph{Pre-trained Distillation Guidance (PDG).}
\begin{figure}[]
	\begin{center}
		\includegraphics[trim=100 225 60 150,clip,width=0.5\textwidth]{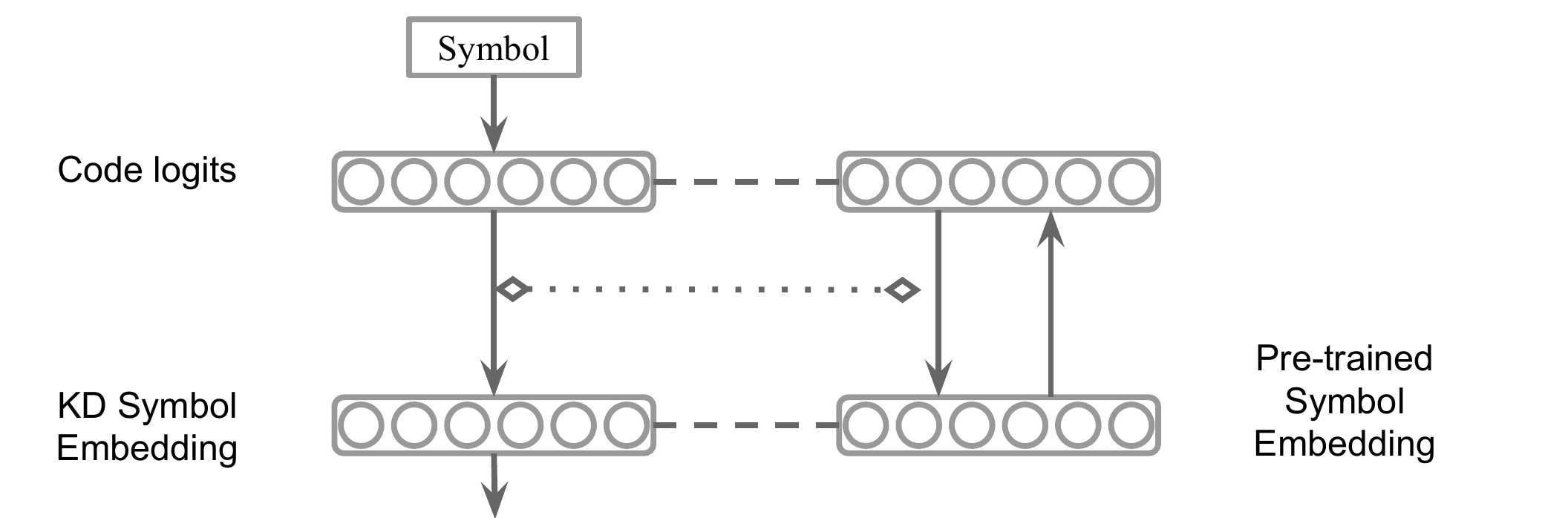}
	\end{center}
	\caption{\label{fig:guidance} Online Distillation Guidance. Dashed lines denotes regularization, doted line in the middle denotes sharing of transformation function.}
\end{figure}
It is important to close the gap between training and inference in the online distillation guidance process, unfortunately this can still be difficult. Alternatively, we can also adopt pre-trained continuous embedding vectors as guidance. Instead of training the continuous embedding vectors alongside the discrete codes, we utilize a pre-trained continuous embedding matrix $\mathcal{U}$ produced from the same model with conventional continuous embedding vectors. During the end-to-end training of the codes (as well as other parameters), we ask the composite embedding vector $\bm v_i$ generated by ``KD encoding'' to mimic the the given embedding vector $\bm u_i$ by minimizing the $l_2$ distance. %This procedure is similar to Knowledge Distillation introduced in \cite{hinton2015distilling}.

Furthermore, we can build an auto-encoder of the pre-trained continuous embedding vectors, and use both continuous embedding vectors as well as the code logits as guidances. In the encoding pass, a transformation function $\bm g(\cdot)$ is used to map $\bm u_i$ to the code logits $\bm{\pi}_i$. In its decoding pass, it utilizes the same transformation function $\bm f(\cdot)$ that is used in ``KD encoding'' to reconstruct $\bm u_i$. The loss function for the auto-encoders is
$$
\mathcal{L}_{auto-encoder} = \sum_i \|\bm f(\bm g(\bm u_i); \tau) - \bm u_i\|^2
$$
To follow the guidance of the pre-trained embedding matrix $\mathcal{U}$, we ask the code logits $\bm \pi_i$ and composite symbol embedding $\bm v_i = \bm f(\bm \pi_i; \tau)$ \footnote{Here we overload the function $\bm f(c_i)$ by considering that code $c_i$ can be turned into ``one-hot'' $\bm o_i$, and $\bm o_i\approx Softmax(\bm \pi_i/\tau)$.} to mimic the ones in the auto-encoder as follows
$$
\mathcal{L}_{distillation} = \sum_{i}\alpha \|\bm f(\bm \pi_i;\tau) - \bm u_i\|^2 + \beta \|\bm \pi_i - \bm g(\bm u_i)\|^2
$$
During the training, both $\mathcal{L}_{auto-encoder}$ and $\mathcal{L}_{distillation}$ will be added to the task-specific loss function to train jointly. The method is illustrated in the Figure \ref{fig:guidance}.

Here we also make a distinction between pre-trained distillation guidance (PDG) and pre-training of codes.  Firstly, PDG can learn codes end-to-end to optimize the task's loss, while the pre-trained codes will be fixed during the task learning. Secondly, the PDG training procedure is much easier, especially for the tuning of discrete code learning, while pre-training of codes requires three stages and is unfriendly for parameter tuning.
\section{Experiments}

In this section, we conduct experiments to validate the proposed approach. Since the proposed ``KD Encoding" can be applied to various tasks and applications with embedding layers involved. We choose three important tasks for evaluation, they are (1) language modeling, (2) text classification, and (3) graph convolutional networks for semi-supervised node classification. For the detailed descriptions of these tasks and other applications of our method, we refer readers to the supplementary material.

For the language modeling task, we test on the widely used English Penn Treebank \cite{marcus1993building} dataset, which contains 1M words with vocabulary size of 10K. The training/validation/test split is provided by convention according to \cite{mikolov2010recurrent}. Since we only focus on the embedding layer, we simply adopt a previous state-of-the-art model \cite{zaremba2014recurrent}, in which they provide three different variants of LSTMs \cite{hochreiter1997long} of different sizes: The larger model has word embedding size and LSTM hidden size of 1500, while the number is 650 and 200 for the medium and small models. By default, we use $K=32, D=32$ and pre-trained distillation guidance for the proposed method, and linear embedding transformation function with 1 hidden layer of 300 hidden units.

For the text classification task, we utilize five different datasets from \cite{zhang2015character}, namely Yahoo! news, AG's news, DBpedia, Yelp review polarity ratings as well Yelp review full-scale ratings \footnote{YahooAnswers has 477K unique words and 131M tokens, and Yelp has 268K unique words and 94M tokens. More details available in \cite{zhang2015character}.}. We adopt network architecture used in FastText \cite{joulin2016bag,joulin2016fasttext}, where a SoftMax is stacked on top of the averaged word embedding vectors of the text. For simplicity, we only use unigram word information but not sub-words or bi-grams, as used in their work. The word embedding dimension is chosen to be 300 as it yields a good balance between size and performance. By default, we use $K=32, D=32$ for the proposed method, and linear transformation with no hidden layer. That is to add code embedding vectors together to generate symbol embedding vector, and the dimension of code embedding is the same as word embedding.

For the application with graph convolutional networks, we follow the same setting and hyper-parameters as in \cite{kipf2016semi}. Three datasets are used for comparison, namely Cora, Citeseer, Pubmed. Since both the number of symbols (1433, 3703, and 500 respectively) as well as its embedding dimension (16) are small, the compressible space is actually quite small. Nevertheless, we perform the proposed method with $K=64, D=8$ for Cora and Citeseer, and $K=32, D=4$ for Pubmed. Again, a linear embedding transformation function is used with one hidden layer of size 16. We do not use guidances for text classification and graph node classification tasks since the direct optimization is already satisfying enough.

\begin{table}[!t]
	\centering
	\small
	\caption{Language modeling (PTB). Compared with Conventional full embedding, and low-rank (denoted with Lr) with different compression rates.}
	\label{tab:base_lm}
	\begin{tabular}{llcccc}
		\hline\hline
		& Model  & Full   & Lr(5X)         & Lr(10X)      & Ours             \\ \hline
		\multirow{3}{*}{Perplexity}                                                   & Small  & 114.53 & 134.01         & 134.89        & \textbf{107.77} \\
		& Medi. & 83.38  & 84.84          & 85.53         & \textbf{83.11}  \\
		& Large  & 78.71  & 81.23          & 81.85          & \textbf{77.72}  \\ \hline
		\multirow{3}{*}{\begin{tabular}[c]{@{}c@{}}\# of emb.\\ params. \\(M)\end{tabular}} & Small  & 2.00   & 0.40           & \textbf{0.19} & 0.37            \\
		& Medi. & 6.50   & 1.30           & 0.65          & \textbf{0.50}   \\
		& Large  & 15.00  & 2.99           & 1.50          & \textbf{0.76}   \\ \hline
		\multicolumn{1}{l}{\multirow{3}{*}{\begin{tabular}[c]{@{}c@{}}\# of bits\\ (M)\end{tabular}}}                                 & Small  & 64.00  & 12.73 & \textbf{6.20}          & 13.39           \\
		\multicolumn{1}{r}{}                                                          & Medi. & 208.00 & 41.58          & 20.79         & \textbf{17.75}  \\
		\multicolumn{1}{r}{}                                                          & Large  & 480.00 & 95.68          & 47.84         & \textbf{26.00}  \\ \hline
		\hline
	\end{tabular}
\end{table}
\begin{table}[!t]
	\centering
	\small
	\caption{Text classification. Lr denotes low-rank.}
	\label{tab:base_text}
	\begin{tabular}{llcccc}
		\hline\hline
		& Model   & Full           & Lr(10X) & Lr(20X) & Ours            \\ \hline
		\multirow{5}{*}{Accuracy}                                                   & Yahoo!  & \textbf{0.698} & 0.695  & 0.691  & 0.695          \\
		& AG N. & 0.914          & 0.914  & 0.915  & \textbf{0.916} \\
		& Yelp P. & \textbf{0.932} & 0.924  & 0.923  & 0.931          \\
		& Yelp F. & \textbf{0.592} & 0.578  & 0.573  & 0.590          \\
		& DBpedia & 0.977          & 0.977  & 0.979  & \textbf{0.980} \\ \hline
		\multirow{5}{*}{\begin{tabular}[c]{@{}c@{}}\# of emb.\\ params. \\ (M)\end{tabular}} & Yahoo!  & 143.26       & 13.857 & 6.690  & \textbf{0.308} \\
		& AG N. & 20.797         & 2.019  & 0.975  & \textbf{0.308} \\
		& Yelp P. & 74.022         & 7.164  & 3.459  & \textbf{0.308} \\
		& Yelp F. & 80.524         & 7.793  & 3.762  & \textbf{0.308} \\
		& DBpedia & 183.76        & 17.772 & 8.580  & \textbf{0.308} \\ \hline
		\multicolumn{1}{l}{\multirow{3}{*}{\begin{tabular}[c]{@{}c@{}}\# of bits\\ (G)\end{tabular}}}                               & Yahoo!  & 4.584          & 0.443  & 0.214  & \textbf{0.086} \\
		\multicolumn{1}{r}{}                                                          & AG N. & 0.665          & 0.065  & 0.031  & \textbf{0.021} \\
		\multicolumn{1}{r}{}                                                          & Yelp P. & 2.369          & 0.229  & 0.111  & \textbf{0.049} \\
		\multicolumn{1}{r}{}                                                          & Yelp F. & 2.577          & 0.249  & 0.120  & \textbf{0.053} \\
		\multicolumn{1}{r}{}                                                          & DBpedia & 5.880          & 0.569  & 0.275  & \textbf{0.108} \\ \hline
		\hline
	\end{tabular}
\end{table}
We mainly compare the proposed ``KD encoding'' approach with the conventional continuous (full) embedding counterpart, and also compare with low-rank factorization \cite{sainath2013low} with different compression ratios. The results for three tasks are shown in Table \ref{tab:base_lm}, \ref{tab:base_text}, \ref{tab:base_gcn}, respectively. In these tables, three types of metrics are shown: (1) the performance metric, perplexity for language modeling and accuracy for the others, (2) the number of embedding parameters $\theta$ used in $\bm f$, and (3) the total embedding layer's size includes $\theta$ as well as the codes. From these tables, we observe that the proposed ``KD encoding'' with end-to-end code learning perform similarly, or even better in many cases, while consistently saving more than 90\% of embeding parameter and model size, 98\% in the text classification case. In order to achieve similar level of compression, we note that low-rank factorization baseline will reduce the performance significantly.
\begin{table}[!t]
	\centering
	\small
	\caption{Graph Convolutional Networks. Lr denotes low-rank.}
	\label{tab:base_gcn}
	\begin{tabular}{lccccc}
		\hline
		\hline
		&    Dataset     & Full   & Lr(2X)         & Lr(4X) & Ours        \\ \hline
		\multirow{3}{*}{Accuracy}                                                         & Cora    & 0.814  & 0.789          & 0.767 & \textbf{0.823} \\
		& Citese. & 0.721  & 0.710          & 0.685 & \textbf{0.723} \\
		& Pubm.   & 0.795  & 0.773          & 0.780 & \textbf{0.797} \\ \hline
		\multirow{3}{*}{\begin{tabular}[c]{@{}c@{}}\# of emb.\\ params.\\ (K)\end{tabular}} & Cora    & 22.93 & 10.14         & \textbf{5.8} & 8.22 \\
		& Citese. & 59.25 & 26.03         & 14.88 & \textbf{8.22} \\
		& Pubm.   & 8.00  & 3.61          & \textbf{2.06} & 2.69 \\ \hline
		\multicolumn{1}{l}{\multirow{3}{*}{\begin{tabular}[c]{@{}c@{}}\# of bits\\ (M)\end{tabular}}}                               & Cora    & 0.73  & 0.32 & \textbf{0.19} & 0.33          \\
		\multicolumn{1}{l}{}                                                          & Citese. & 1.90  & 0.83          & 0.48 & \textbf{0.44} \\
		\multicolumn{1}{r}{}                                                          & Pubm.   & 0.26  & 0.12          & \textbf{0.07} & 0.10 \\ \hline
		\hline
	\end{tabular}
\end{table}

We further compare with broader baselines on language modeling tasks (with medium sized language model for convenience): (1) directly using first 10 chars of a word as its code (padding when necessary), (2) training aware quantization \cite{jacob2017quantization}, and (3) product quantization \cite{jegou2011product,joulin2016fasttext}. The results are shown in Table \ref{tab:baselines}. We can see that our methods significantly outperform these baselines, in terms of both PPL as well as model size (bits) reduction.

\begin{table}[t]
	\small
\centering
\caption{Comparisons with more baselines in Language Modeling (Medium sized model).}
\label{tab:baselines}
\begin{tabular}{lcc}
\hline
\hline
Methods                       & PPL & Bits saved \\ \hline
Char-as-codes                 & 108.14     & 96\%         \\
Scalar quantization (8 bits)  & 84.06      & 75\%         \\
Scalar quantization (6 bits)  & 87.73      & 81\%         \\
Scalar quantization (4 bits)  & 92.86      & 88\%         \\
Product quantization(64x325)  & 84.03      & 88\%         \\
Product quantization(128x325) & 83.71      & 85\%         \\
Product quantization(256x325) & 83.66      & 81\%         \\ \hline
Ours                          & 83.11      & 92\%         \\ \hline \hline
\end{tabular}
\end{table}
In the following, we scrutinize different components of the proposed model based on PTB language modeling. To start with, we test various code learning methods, and demonstrate the impact of training with guidance. The results are shown in Table \ref{tab:diff_coding}. First, we note that both random codes as well as pre-trained codes are suboptimal, which is understandable as they are not (fully) adaptive to the target tasks. Then, we see that end-to-end training without guidance suffers serious performance loss, especially when the task specific networks increase its complexity (with larger hidden size and use of dropout). Finally, by adopting the proposed continuous guidances (especially distillation guidance), the performance loss can be overcame.
\begin{table}[!t]
	\small
	\centering
	\caption{Comparisons of different code learning methods.}
	\label{tab:diff_coding}
	\begin{tabular}{llcc}
		\hline\hline
		& Small     & Medium   & Large    \\ \hline
		Full embedding         & 114.53          & 83.38        & 78.71          \\\hline
		Random code         & 115.79          & 104.12         & 98.38          \\
		Pre-trained code       & 107.95          & 84.92          & 80.69          \\\hline
		Ours (no guidance) & 108.50          & 89.03          & 86.41          \\
		Ours (ODG)  & 108.19          & 85.50         & 83.00          \\
		Ours (PDG)  & \textbf{107.77} & \textbf{83.11} & \textbf{77.72} \\ \hline
		\hline
	\end{tabular}
\end{table}
We further vary the size of $K$ or $D$ and see how they affect the performance. As shown in Figure \ref{fig:kd_mean_ppl} and \ref{fig:kd_rnn_ppl}, small K or D may harm the performance (even though that $K^D \gg N$ is satisfied), which suggests that the redundant code can be easier to learn. The size of $D$ seems to have higher impact on the performance compared to $K$. Also, when $D$ is small, non-linear encoder such as RNN performs much better than the linear counterpart, which verifies our Proposition 2.
\begin{figure}[!t]
	\small
	\centering
	\begin{subfigure}[b]{0.24\textwidth}
		\includegraphics[width=\textwidth]{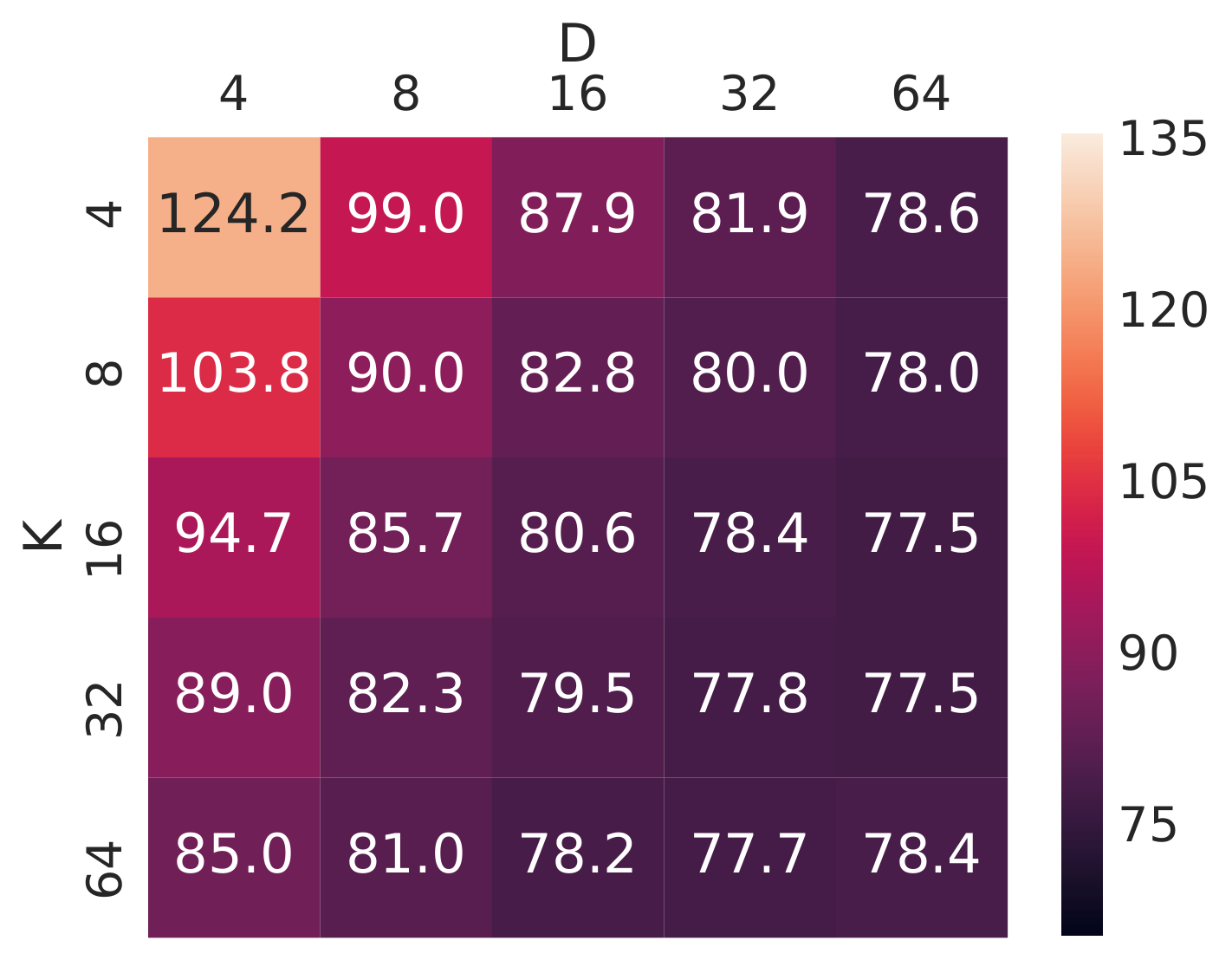}
		\caption{Linear instantiation.}
		\label{fig:kd_mean_ppl}
	\end{subfigure}%
	~ 
	\begin{subfigure}[b]{0.24\textwidth}
		\includegraphics[width=\textwidth]{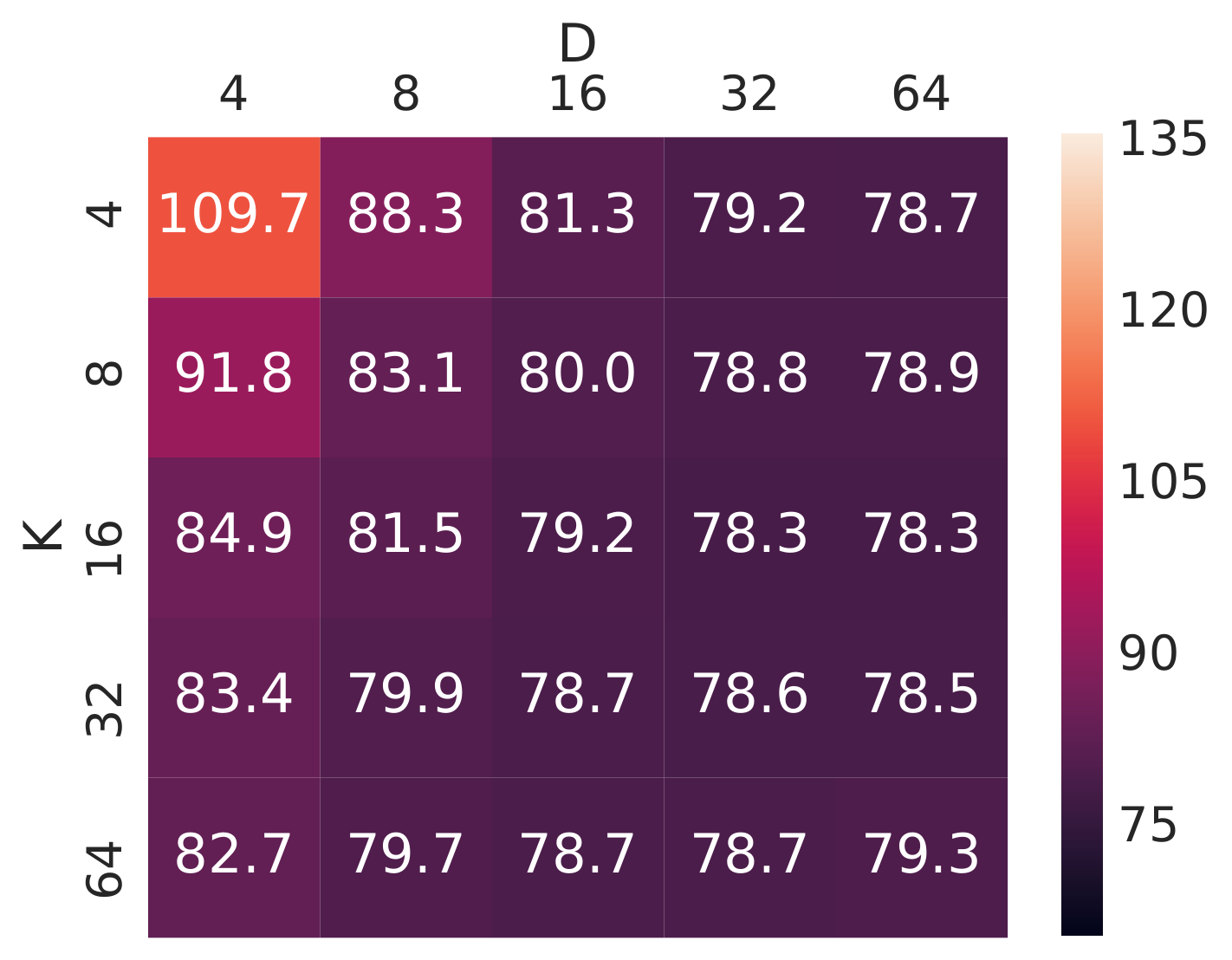}
		\caption{RNN instantiation.}
		\label{fig:kd_rnn_ppl}
	\end{subfigure}
	\caption{The effects of various K and D under different instantiation of embedding transformation function $\bm f(\cdot)$.}\label{fig:kd}
\end{figure}
\begin{table}[!t]
\centering
\small
\caption{Learned codes for 10K Glove embeddings (K=6, D=4).}
\label{tab:code_demo}
\begin{tabular}{lp{20em}}\Xhline{2\arrayrulewidth}
	Code    & Words\\ \hline
	3-1-0-3 & up when over into time back off set left open half behind quickly starts\\
	3-1-0-4 & week tuesday wednesday monday thursday friday sunday saturday \\
	3-1-0-5 & by were after before while past ago close soon recently continued meanwhile \\
	3-1-1-1 & year month months record fall annual target cuts      \\\Xhline{2\arrayrulewidth}
\end{tabular}
\end{table}
To examine the learned codes, we apply our method on the pre-trained embedding vectors from Glove \cite{pennington2014glove}, which has better coverage and quality. We force the model to assign multiple words to the same code by setting $K=6, D=4$ (code space is 1296) for vocabulary size of 10K. Table \ref{tab:code_demo} show a snippet of the learned codes, which shows that semantically similar words are assigned to the same or close-by discrete codes.
\section{Related Work}

The idea of using more efficient coding system traces to information theory, such as error correction code \cite{hamming1950error}, and Hoffman code \cite{huffman1952method}. However, in most embedding techniques such as word embedding \cite{mikolov2013distributed,pennington2014glove}, entity embedding \cite{chen2016entity,chen2017task}, ``one-hot'' encoding is used along with a usually large embedding matrix. Recent work \cite{kim2016character,sennrich2015neural,zhang2015character} explores character or sub-word based embedding model instead of the word embedding model and show some promising results. \cite{svenstrup2017hash} proposes using hash functions to automatically map texts to pre-defined bases with a smaller vocabulary size, according to which vectors are composed. However, in their cases, the chars, sub-words and hash functions are fixed and given a priori dependent on language, thus may have few semantic meanings attached and may not be available for other type of data. In contrast, we learn the code assignment function from data and tasks, and our method is language independent.

The compression of neural networks \cite{han2015deep,han2015learning,chen2015compressing} has become more and more important in order to deploy large networks to small mobile devices. Our work can be seen as a way to compress the embedding layer in neural networks. Most existing network compression techniques focus on dense/convolutional layers that are shared/amortized by all data instances, while one data instance only utilizes a fraction of embedding layer weights associated with the given symbols. To compress these types of weights, some efforts have been made, such as product quantization \cite{jegou2011product,joulin2016fasttext,zhang2014composite,zhang2015sparse,babenko2014additive}. Compared to their methods, our framework is more general. Many of these methods can be seen as a special case of ``KD encoding'' using a linear embedding transformation function without hidden layer. Also, under our framework, both the codes and the transformation functions can be learned jointly by minimizing task-specific losses.

Our work is also related to LightRNN \cite{li2016lightrnn}, which can be seen as a special case of our proposed KD code with $K=\sqrt{N}$ and $D=2$. Due to the use of a more compact code, its code learning is harder and more expensive. This work is an extension of our previous workshop paper \cite{chen2017learning} with guided end-to-end code learning. In parallel to \cite{chen2017learning}, \cite{shu2017compress} explores similar ideas with linear composition functions and pre-trained codes.

\section{Conclusions}
In this paper, we propose a novel K-way D-dimensional discrete encoding scheme to replace the ``one-hot" encoding, which significantly improves the efficiency of the parameterization of models with embedding layers. 
%Furthermore, the reduction of parameters can also mitigate the over-fitting problem in some cases. 
To learn semantically meaningful codes, we derive a relaxed discrete optimization technique based on SGD enabling end-to-end code learning. We demonstrate the effectiveness of our work with applications in language modeling, text classification and graph convolutional networks.
%The embedding layer's size can be reduced up to 98\% while achieving similar or better performance.
% In the future, we plan to further improve the performance and generality of the approach while reducing procedural overhead, and provide an universal tensorflow/pytorch API, such as ``tf.KD\_embedding\_lookup'' to be a compact and competitive alternative to ``tf.embedding\_lookup''.

%\small
\clearpage\newpage
\section*{Acknowledgements}
We would like to thank anonymous reviewers for their constructive comments. We would also like to thank Chong Wang, Denny Zhou, Lihong Li for some helpful discussions. This work is partially supported by NSF III-1705169, NSF CAREER Award 1741634, and Snapchat gift funds.
\bibliography{content/ref}
\bibliographystyle{icml2018}

\newpage
\appendix{\textbf{\Large Appendix}}
\section{Proofs of Lemmas and Propositions}

\begin{lemma}
	The number of embedding parameters used in KD encoding is $O(\frac{K}{\log K} d' \log N + C)$, where $C$ is the number of parameters of neural nets.
\end{lemma}
\begin{proof}
	As mentioned, the embedding parameters include code embedding matrix $\{\mathcal{W}\}$ and embedding transformation function $\theta_e$. There are $O(\frac{K}{\log K} \log N)$ code embedding vectors with $d'$ dimensions. As for the number of parameters in embedding transformation function such as neural networks (LSTM) $C$ that is in $O(d'^2)$, it can be treated as a constant to the number of symbols since $d'$ is independent of $N$, provided that there are certain structures presented in the symbol embeddings. For example, if we assume all the symbol embeddings are within $\epsilon$-ball of a finite number of centroids in $d$-dimensional space, it should only require a constant $C$ to achieve $\epsilon$-distance error bound, regardless of the vocabulary size, since the neural networks just have to memorize the finite centroids.
\end{proof}.

\begin{proposition}
	A linear composition function $\bm f$ with no hidden layer is equivalent to a sparse binary low-rank factorization of the embedding matrix.
\end{proposition}
\textit{Proof sketch}. First consider when $K=2$, and the composed embedding matrix can be written as $U=BC$, where $B$ is the binary code for each symbol, and $C$ is the code embedding matrix. This is a low rank factorization of the embedding matrix with binary code $B$. When we increase $K$, by representing a choice of $K$ as one-hot vector of size $K$, we still have $U=BC$ with additional constraints in B that it is a concatenation of $D$ one-hot vector. Due to the one-hot constraint, each row in $B$ will be sparse as only $1/K$ ratio of entries are non-zero, thus corresponds to a sparse binary low-rank factorization of the embedding matrix.

As the linear composition with no hidden layer can be limited in some cases as the expressiveness of the function highly relies on the number of bases or rank of the factorization. Hence, the non-linear composition may be more appealing in some cases.
\begin{proposition}
	Given the same dimensionality of the ``KD code'', i.e. K, D, and code embedding dimension $d'$, the non-linear embedding transformation functions can reconstruct the embedding matrix with higher rank than the linear counterpart.
\end{proposition}
\textit{Proof sketch.} As shown above, in the linear case, we approximate the embedding by a low-rank factorization, $U=BC$. The rank will be constrained by the dimensionality of binary matrix $B$, i.e. $KD$. However, if we consider a nonlinear transformation function $f$, we will have $U = f(B, C)$. As long as that no two rows in $B$ and no two columns in $C$ are the same, i.e. every data point has its quite code and every code has its unique embedding vector, then the non-linear function $f$, such as a neural network with enough capacity, can approximate a matrix $U$ that has much higher rank, even full rank, than $KD$.

\section{The LSTM Code Embedding Transformation Function}

Here we present more details on the LSTM code embedding transformation function. Assuming the code embedding dimension is the same as the LSTM hidden dimension, the formulation is given as follows.
\begin{gather*}
\bm{t}_j = \sigma(\mathcal{W}^j_{\bm{c}^j} + \bm{h}_{j-1} U_{t} + \bm{b}_t) \\
\bm{i}_j = \sigma(\mathcal{W}^j_{\bm{c}^j} + \bm{h}_{j-1} U_{i} + \bm{b}_i) \\
\bm{o}_j = \sigma(\mathcal{W}^j_{\bm{c}^j} + \bm{h}_{j-1} U_{t} + \bm{b}_t) \\
\bm{m}_j = \bm{t}_j \circ \bm{m}_{j-1} + \bm{i}_j \circ \tanh(\mathcal{W}^j_{\bm{c}^j}+ U_{m} \bm{h}_{j-1} + \bm{b}_m) \\
\bm{h}_j = \bm{o}_j \circ \tanh(\bm{m}_j),
\end{gather*}
where $\sigma(\cdot)$ and $\tanh(\cdot)$ are, respectively, standard sigmoid and tanh activation functions. Please note that the symbol index $i$ is ignored for simplicity. 

\section{Examples and Applications}
Our proposed task-specific end-to-end learned ``KD Encoding" can be applied to any problem involving learning embeddings to reduce model size and increase efficiency.  In the following, we list some typical examples and applications, for which detailed descriptions can be found in the supplementary material. 

\paragraph{Language Modeling} 
Language modeling is a fundamental problem in NLP, and it can be formulated as predicting the probability over a sequence of words. Models based on recurrent neural networks (RNN) with word embedding \cite{mikolov2010recurrent,kim2016character} achieve state-of-the-art results, so on which we will base our experiments. A RNN language model estimates the probability distribution of a sequence of words by modeling the conditional probability of each word given its preceding words, 
\begin{equation}
P(w_0,...,w_N) = P(w_0) \prod_{i=1}^N P(w_i|w_0,...,w_{i-1}), 
\end{equation}
where $w_i$ is the $i$-th word in a vocabulary, and the conditional probability $P(w_i|w_0,...,w_{i?1})$ can be naturally modeled by a softmax output at the $i$-th time step of the RNN. The RNN parameters and the word embeddings are model parameters of the language model.
\paragraph{Text Classification} 
Text classification is another important problem in NLP with many different applications. In this problem, given a training set of documents with each containing a number of words and its target label, we learn the embedding representation of each word and a binary or multi-class classifier with a logistic or softmax output, predicting the labels of test documents with the same vocabulary as in the training set. To test the ``KD Encoding" of word embedding on several typical text classification applications, we use several different types of datasets: Yahoo answer and AG news represent topic prediction, Yelp Polarity and Yelp Full represent sentiment analysis, while DBpedia represents ontology classification.
%Document classification is another important problem in NLP with many different applications. In this problem, given a training set of documents with each containing a number of words and its target label, we learn the embedding representation of each word and a binary or multi-class classifier with a logistic or softmax output, predicting the labels of test documents with the same vocabulary as in the training set.
\paragraph{Graph Convolutional Networks for Semi-Supervised Node Classification}
In \cite{kipf2016semi}, graph convolutional networks (GCN) are proposed for semi-supervised node classification on undirected graphs. In GCN, the matrix based on standard graph adjacency matrix with added self connections after normalization, $\hat{A}$, is used to approximate spectral graph convolutions. As a result, $ReLU(\hat{A} X W)$ defines a non-linear convolutional feature transformation on node embedding matrix $X$ with a projection matrix $W$ and non-linear activation function $ReLU$. This layer-wise transformation can be repeated to build a deep network before making predictions using the final output layer. Minimizing a task-specific loss function, the network weights $W$s and the node embedding matrix $X$ are learned simultaneously using standard back-propagation. A simple GCN with one hidden layer takes the following form:
\begin{equation}
Z = f(X, A) = \textnormal{softmax} (\hat{A} \textnormal{ReLU}( \hat{A}XW_0) W_1),\\
\end{equation}
where $W_0$ and $W_1$ are network weights, and softmax is performed in a row-wise manner. When the labels of only a subset of nodes are given, this framework is readily extended for graph-based semi-supervised node classification by minimizing the following loss function,
\begin{equation}
L = -\sum_{l=1}^L \sum_{f=1}^F Y_{lf} \ln Z_{lf} ,
\end{equation}
where $L$ is the number of labeled graph nodes, $F$ is the total number of classes of the graph nodes, and $Y$ is a binary label matrix with each row summing to $1$. We apply our proposed KD code learning to graph node embeddings in the above GCN framework for semi-supervised node classification.
\paragraph{Hashing}
The learned discrete code can also be seen as a data-dependent hashing for fast data retrieval. In this paper, we also perform some case studies evaluating the effectiveness of our learned KD code as hash code.

\section{Additional Experimental Results}
We also test the effects of different code embedding dimensions, and the result is presented in Figure \ref{fig:code_emb_dim}. We found that linear encoder requires larger code embedding dimensionality, while the non-linear encoder can work well with related small ones. This again verifies the proposition 2.
\begin{figure}[!t]
	\begin{center}
		\includegraphics[width=0.3\textwidth]{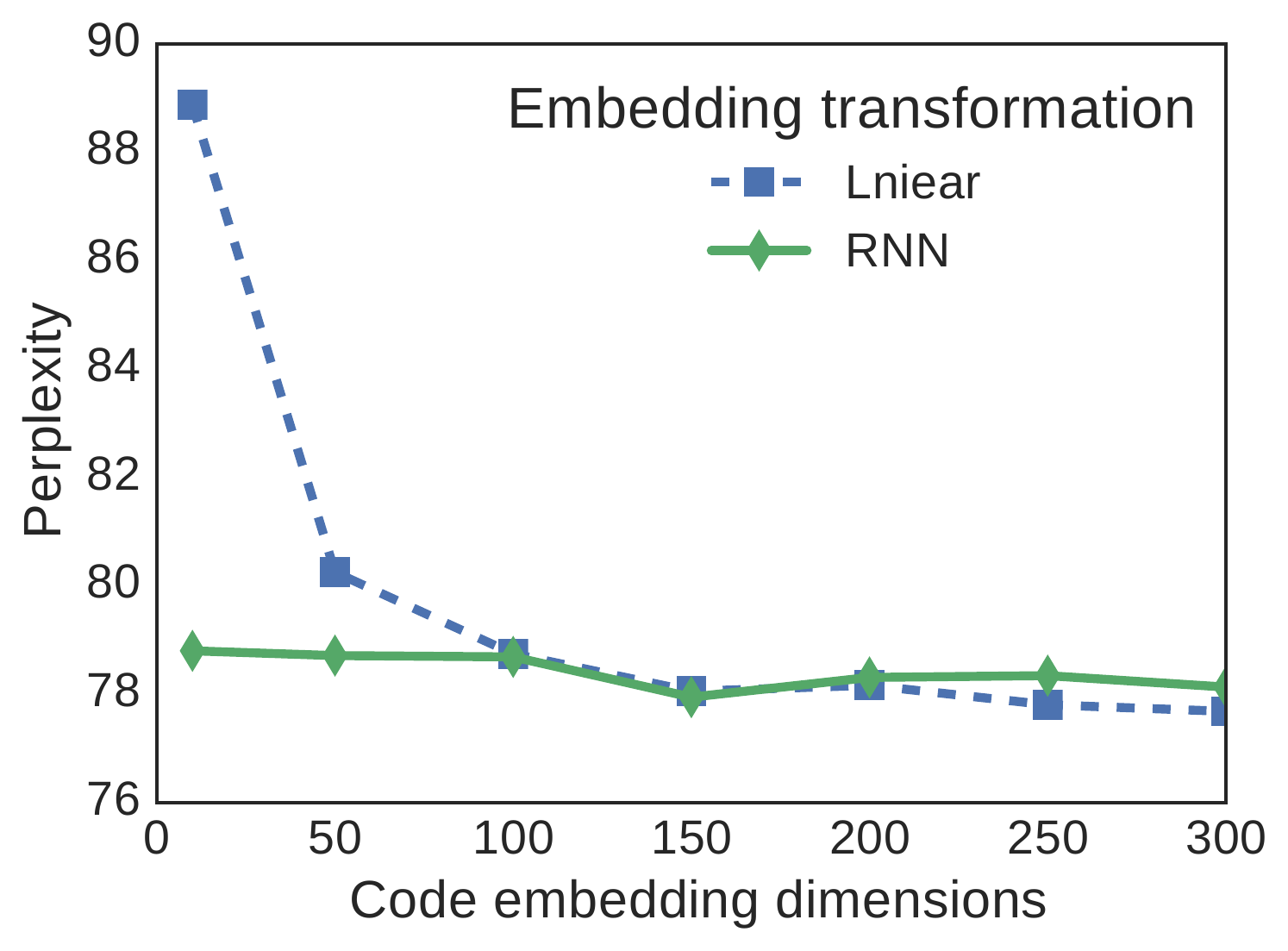}
	\end{center}
	\caption{\label{fig:code_emb_dim} The perplexity on PTB as a function of different code embedding dimensions as well as the embedding transformation functions.}
\end{figure}

Table \ref{tab:opt_tricks} shows the effectiveness of variants of the tricks in continuous relaxation based optimization. We can clearly see that the positive impacts of temperature scheduling, and/or entropy regularization, as well as the auto-encoding. However, here the really big performance jump is brought by using the proposed distillation guidance.

\begin{table}[!t]
\centering
\caption{Effectiveness of different optimization tricks. Here, CR=Continuous Relaxation using softmax, STE=straight-through estimation, CDG=continuous distillation guidance.}
\label{tab:opt_tricks}
\begin{tabular}{ll}
\hline \hline
Variants                                        & PPL   \\ \hline
CR                                              & 90.61 \\
CR + STE                                        & 90.15 \\
CR + STE + temperature scheduling               & 89.55 \\
CR + STE + entropy reg                          & 89.03 \\
CR + STE + entropy reg + PDG (w/o autoencod.) & 83.71 \\
CR + STE + entropy reg + PDG (w/ autoencod.)  & 83.11 \\ \hline \hline
\end{tabular}
\end{table}

% We also test our methods on text8 dataset (with same hyper-parameters used in PTB, without any tuning), the baseline PPL is the 131.29, and KD code based embedding achieves 133 with online distillation and 134.98 with pre-trained distillation. We believe with further tuning these results can be better improved.

\section{Notations}
For clarity, Table 1 provides explanations to major notations used in our paper.

\newpage
\begin{table}[H]
	\centering
	\label{tab:notations}
	\caption{Notations}
	\begin{tabular}{cp{17em}}
		\hline
		Notations & Explanation \\ \hline\hline
		$\bm c$& Codes.           \\
		$\bm{o}$& One-hot representations of the code.             \\
		$\bm{\hat{o}}$& Continuously relaxed $\bm o$.            \\
		$\bm{\pi}$& code logits for computing $\bm{\hat{o}}$.            \\
		$\mathcal{W}$& Code embedding matrix.             \\     
		$\mathcal{T}$& The transformation from symbol to the embedding , $\mathcal{T}=\bm f \circ \bm\phi .$\\
		$\bm{\phi}$& The transformation from symbol to code. \\
		$\bm{f}$& The code transformation function maps code to embedding. It has parameters $\theta = \{\mathcal{W}, \theta_e\}$\\
		$\bm{f}_e$& The embedding transformation function maps code embedding vectors to a symbol embedding vector. \\
		$\bm{v}$& The composite symbol embedding vector. \\
		$\Theta$& The task-specific (non-embedding) parameters.\\
		$\mathcal{U}$& Pre-trained symbol embedding matrix.  \\
		$\bm{u}$& Pre-trained symbol embedding vector.  \\  
		$d$& Symbol embedding dimensionality. \\
		$d'$& Code embedding dimensionality. \\
		\hline
	\end{tabular}
\end{table}

\end{document}